\documentclass{article}

\usepackage[preprint]{neurips_2025}
\usepackage{amsthm}
\usepackage{amsmath}
\usepackage{graphicx}

\makeatletter
\def\@noticestring{}
\makeatother

\newtheorem{theorem}{Theorem}
\newtheorem{lemma}{Lemma}
\newtheorem{remark}{Remark}
\newtheorem{definition}{Definition}
\newtheorem{example}{Example}
\usepackage{algorithmicx}
\usepackage{algorithm}
\usepackage{algpseudocode}

\usepackage[utf8]{inputenc} 
\usepackage[T1]{fontenc}    
\usepackage{hyperref}       
\usepackage{url}            
\usepackage{booktabs}       
\usepackage{amsfonts}       
\usepackage{nicefrac}       

\usepackage{microtype}      
\usepackage{xcolor}         

\title{Riemannian Optimization in Modular Systems}

\author{
  Christian Pehle\\
  Cold Spring Harbor Laboratory\\
  Cold Spring Harbor, NY 11724 \\
  \texttt{pehle@cshl.edu} \\
  \And
  Jean-Jacques Slotine \\
  Massachusetts Institute of Technology\\ 
  Cambridge, MA 02139\\
  \texttt{jjs@mit.edu}
}

\begin{document}

\maketitle

\begin{abstract}
Understanding how systems built out of modular components can be jointly optimized is an important problem in biology, engineering, and machine learning. The backpropagation algorithm is one such solution and has been instrumental in the success of neural networks. Despite its empirical success, a strong theoretical understanding of it is lacking. Here, we combine tools from Riemannian geometry, optimal control theory, and theoretical physics to advance this understanding. 

We make three key contributions: First, we revisit the derivation of backpropagation as a constrained optimization problem and combine it with the insight that Riemannian gradient descent trajectories can be understood as the minimum of an action. Second, we introduce a recursively defined layerwise Riemannian metric that exploits the modular structure of neural networks and can be efficiently computed using the Woodbury matrix identity, avoiding the $O(n^3)$ cost of full metric inversion. Third, we develop a framework of composable ``Riemannian modules'' whose convergence properties can be quantified using nonlinear contraction theory, providing algorithmic stability guarantees of order $O(\kappa^2 L/(\xi \mu \sqrt{n}))$ where $\kappa$ and $L$ are Lipschitz constants, $\mu$ is the mass matrix scale, and $\xi$ bounds the condition number.

Our layerwise metric approach provides a practical alternative to natural gradient descent. While we focus here on studying neural networks, our approach more generally applies to the study of systems made of modules that are optimized over time, as it occurs in biology during both evolution and development.
\end{abstract}

\section{Introduction}

\emph{Modularity} is the hallmark of many complex systems. It can be found in biological organisms, engineered machines, and machine learning models. In these systems, individual components -- or modules -- have parameters that can be optimized independently, yet their interactions must be taken into account to achieve overarching system goals. Within machine learning, neural networks serve as a prime example of modular systems, where layers function as modules with parameters that are collectively tuned to minimize a loss function. The backpropagation algorithm has been a cornerstone of this process, driving the remarkable success of neural networks. Despite its empirical success, a firm theoretical understanding is lacking.

This paper addresses this gap by synthesizing insights from Riemannian geometry, optimal control theory, and theoretical physics to advance our understanding of optimization in modular systems. We reframe the optimization of such systems as a constrained problem on a Riemannian manifold, where the manifold's geometry mirrors the modular structure. This revisits the formulation of backpropagation as a constrained optimization problem \cite{lecun1988theoretical} originally formulated in the context of optimal control theory \cite{bryson1969applied} and connects it to the observation that gradient descent paths can be understood as minima of an action \cite{witten1982supersymmetry}.

This approach not only recovers backpropagation as a special case but also equips us with novel tools to analyze and enhance optimization strategies in modular systems.

\section{Action Principle for Gradient Descent}

\paragraph{Intuition.} The key insight connecting physics and optimization is that gradient descent trajectories can be understood as paths that minimize a certain ``action'' - a concept from field theory. Unlike classical mechanics where action balances kinetic and potential energy, here both terms in the action are quadratic forms with respect to the Riemannian metric: the first term penalizes rapid parameter changes (measured by the metric $g_{IJ}$), while the second term penalizes large gradients (measured by the inverse metric $g^{IJ}$). This formulation, inspired by Witten's supersymmetric quantum mechanics \cite{witten1982supersymmetry}, reveals that gradient descent paths are those that optimally balance parameter velocity against gradient magnitude, both weighted by the geometry of the parameter space.

Gradient flow can be seen as the critical point of an action. Given a Riemannian manifold $(M,g)$ with coordinate functions $\phi^I$ (where $I,J = 1,\ldots,n$ index coordinates) and a smooth map $h \colon M \to \mathbf{R}$, the action is given by
\begin{equation}
\label{eq:action-gradient-descent}
S = \frac{1}{2} \int ds \left(g_{IJ} \frac{d \phi^I}{ds} \frac{d \phi^J}{ds} + \eta^2 g^{IJ} \frac{\partial h}{\partial \phi^I} \frac{\partial h}{\partial \phi^J}\right)
\end{equation}
where $g_{IJ}$ are the metric components, $g^{IJ} = (g^{-1})^{IJ}$ are the inverse metric components, and Einstein summation convention is used.
This is easy to see by completing the square and integrating by parts:
\begin{equation*}
S = \frac{1}{2} \int ds \ g_{IJ} \left(\frac{d \phi^I}{ds} \pm \eta g^{IK} \frac{\partial h}{\partial \phi^K}\right) \left(\frac{d \phi^J}{ds} \pm \eta g^{JL} \frac{\partial h}{\partial \phi^L}\right) \mp \eta \int ds \frac{dh}{ds}
\end{equation*}
For any trajectory, this action is bounded from below by
\begin{equation*}
S \geq \eta \vert  h(s=+\infty) - h(s=-\infty) \vert
\end{equation*}
and equality is reached iff
\begin{equation*}
\frac{d \phi^I}{ds} \pm \eta g^{IJ} \frac{\partial h}{\partial \phi^J} = 0.
\end{equation*}
These are precisely the equations of Riemannian gradient descent and ascent. We will now compute some immediate properties of this Lagrangian and then turn to how we regard neural network training from this perspective.

\begin{theorem}The Hamiltonian $H$ associated with the Lagrange density
\begin{equation*}
L = \frac{1}{2} g_{IJ} \frac{d \phi^I}{ds} \frac{d \phi^J}{ds} + \frac{1}{2}\eta^2 g^{IJ} \partial_I h \partial_J h
\end{equation*}
vanishes along gradient descent and ascent paths.
\end{theorem}
\begin{proof}
The canonical momentum associated with $L$ is
\begin{equation*}
p_I = \frac{\partial L}{\partial \dot{\phi}^I} = g_{IJ} \dot{\phi}^J
\end{equation*}
The Hamiltonian is then
\begin{equation*}
H = p_I \dot{\phi}^I - L = \frac{1}{2} g^{IJ} p_I p_J - \frac{1}{2} \eta^2 g^{IJ} \partial_I h \partial_J h = \frac{1}{2} g^{IJ}(p_I - \eta \partial_I h)(p_J + \eta \partial_J h)
\end{equation*}
This shows that on gradient descent and ascent paths the Hamiltonian vanishes.
\end{proof}

\begin{theorem}
The Hamilton equations on $T^* M$ for the Hamiltonian $H$ are
\begin{align*}
\dot{\phi}^I &= \frac{\partial H}{\partial p_I} = g^{IJ} p_J\\
\dot{p}_I &= - \frac{\partial H}{\partial \phi^I} = -\frac{1}{2} g^{KL}_{,I} p_K p_L + \frac{1}{2} \eta^2 g^{KL}_{,I} \partial_K h \partial_L h + \eta^2 g^{KL} (\partial_I \partial_K h) \partial_L h
\end{align*}
with $g^{KL}_{,I} = \partial_I g^{KL}$.
\end{theorem}
\begin{remark}In the limit $\eta \to 0$ these define the co-geodesic flow. Similarly the Euler-Lagrange equations for the Lagrangian $L$ in the limit $\eta \to 0$ are the geodesic flow.
\end{remark}

The gradient trajectories in a Neural Network can be understood from this perspective by explicitly introducing constraints that correspond to the decomposition of $h$ into a sequence of function applications.

\subsection{Neural Network Case}

To apply the action formulation to the case of a neural network, we can consider the output of the network to be defined by a function $h$. The manifold $M$ is then the space of the network parameters.

The empirical risk minimization problem for a dataset $D = \{(x_i, y_i)\}$ is given by
\begin{equation*}
\mathrm{min}_\phi h(\phi) = \mathrm{min}_\phi \ \frac{1}{n} \sum_i l(f(x_i;\phi),y_i)
\end{equation*}
where $l$ is the loss function and $f$ is the neural network.

In the cases in which we are interested, the function $h$ is the result of the composition of several parametrized functions (a neural network). Calculating the derivatives $\frac{\partial h}{\partial \phi^I}$ directly becomes costly. We can take advantage of the fact that $h$ is the composition of functions by introducing explicit constraints \cite{lecun1988theoretical}, see also \cite{bryson1969applied} (Section 2.1 and 2.2). To do so we can introduce the Lagrangian $L$, writing the equation $h = l(f(x;\phi),y)$ as a series of constraints:
\begin{equation*}
L = l(z^{(l)}, y) - \sum_\alpha \lambda^\top_\alpha (z^{(\alpha)} - f^{(\alpha)}(z^{(\alpha-1)}, \phi_\alpha)).
\end{equation*}
A short calculation then reveals that the derivative of $h$ with respect to the parameters can be expressed in terms of the derivatives of the functions $f^{(\alpha)}$ and the Lagrange multipliers $\lambda_\alpha$:
\begin{equation*}
\frac{\partial h}{\partial \phi^I} = \frac{\partial l(f(x;\phi),y)}{\partial \phi^I} = \left(\frac{\partial f^{(\alpha)}(x;\phi)}{\partial \phi^I}\right)^\top \lambda_\alpha.
\end{equation*}
For simplicity, we have assumed that $\phi^I$ occurs only in $f^{(\alpha)}$, that is, $(\alpha = \mathrm{layer}(I))$. This is the pullback of the cotangent vector $\lambda_\alpha$ (in the activity space of layer $\alpha$) to the weight space of that layer.
Inserting in the action \eqref{eq:action-gradient-descent}, this results in
\begin{equation*}
S = \frac{1}{2} \int ds \left(g_{IJ} \frac{d \phi^I}{ds} \frac{d \phi^J}{ds} + \sum_{\alpha \beta} \eta^2 g^{IJ} \left[ \left(\frac{\partial f^{(\alpha)}(x;\phi)}{\partial \phi^I}\right)^\top \lambda_\alpha \right] \left[ \left(\frac{\partial f^{(\beta)}(x;\phi)}{\partial \phi^J}\right)^\top \lambda_\beta \right]\right)
\end{equation*}
By the same argument as above we then get that the critical points are
\begin{equation*}
\frac{d \phi^I}{ds} \pm \eta g^{IJ} \left[ \left(\frac{\partial f^{(\alpha)} (x;\phi)}{\partial \phi^J}\right)^\top \lambda_\alpha \right]= 0
\end{equation*}

We have derived the gradient descent update as a critical point of an action principle. This naturally raises the question of what the metric $g$ should be. We consider a layerwise metric that is induced by the network architecture itself. In this way, the optimization trajectories can take into account the intrinsic geometry induced by the layered structure.

\begin{remark}{Brayton-Moser equations}
A similar analysis can be carried out for the Brayton-Moser equations. The Brayton-Moser equations describe the dynamics of a system in terms of a potential function $P(\phi, \lambda)$, where $\phi$ represents state variables and $\lambda$ represents co-state or control variables. The system is equipped with two positive definite metrics:
\begin{itemize}
    \item $G_\phi$ a metric on the manifold of state variables $\phi$
    \item $G_\lambda$ a metric on the manifold of co-state variables $\lambda$
\end{itemize}
The action functional \eqref{eq:action-gradient-descent} can be defined analogously.
\end{remark}

\begin{remark}{Differential Algebraic Equations}
The results above can also be extended to the case of a system of differential algebraic equations on some manifold. This is closely related to their control theoretic origin. In this case time $t$ and pseudo-time $s$ both enter the action.
\end{remark}

\section{Riemannian Metrics on Neural Networks}

A priori there is not a unique choice for a Riemannian structure for a given neural network. Here we focus on metrics which are \emph{modular}, that is they are definable on parts of the network and the behaviour of a system composed of those modules can be analysed.

\subsection{Pullback Metric Geometry for Neural Networks}
We will consider a layerwise metric which is induced by the network architecture itself. A given neural network then has a naturally defined Riemannian geometry associated with it, due to its layered structure. In particular, a notion of distance on the output space of the network induces a notion of distance on each of the intermediate layers and weights.

This metric is defined on each layer by pulling back a metric from the output space. The metric is not positive definite, in general, but it is a natural choice to define the notion of distance between points in the space of activations and weights. Intuitively it relates measurements of angles and directions in the output space to angles and directions in all of the intermediate activation and parameter spaces of the network.

We will also assume that on each parameter space of the network there is a Riemannian metric defined on the parameter space. The overall Riemannian metric on a given layer is then the sum of the pullback metric and the layerwise metric. This is in contrast to \cite{benfenati2022singular}, who considered just the pullback metric, or to natural gradient descent, which considers the Fisher metric on the overall parameter space.

\begin{definition}[Pullback Metric]Let $f \colon M \to N$ be a smooth map
between two manifolds $M$ and $N$ with metrics $g_M$ and $g_N$. The pullback metric $f^\ast g_N$ on $M$ is defined by the relation
\begin{equation*}
f^\ast g_N(X,Y) = g_N(df(X), df(Y))
\end{equation*}
for all vector fields $X$ and $Y$ on $M$.
\end{definition}

To link this with the formulation in the previous section, we partition the parameters of the network into layers. This is reflected in a partition of the coordinate indices $I, J$. We then have that the metric $g_{IJ}$ has a block diagonal structure $g = \oplus_\alpha g^{(\alpha)}$, with each block metric $g^{(\alpha)}$ defined on the parameter space of layer $\alpha$. In matrix form this is written as
\begin{equation*}
\mathbf{G} = \bigoplus_\alpha \mathbf{G}^{(\alpha)}
\end{equation*}
where $\mathbf{G}^{(\alpha)}$ is the matrix representation of the metric $g^{(\alpha)}$ in the parameter space of layer $\alpha$. So the matrix $\mathbf{G}$ is block diagonal with blocks $\mathbf{G}^{(\alpha)}$. Each layer metric combines the pullback metric from the output space and a layer-specific parameter metric:
\begin{equation*}
\mathbf{G}^{(\alpha)} = {\mathbf{J}^{(\alpha)}}^\top \mathbf{M} \mathbf{J}^{(\alpha)} + \mathbf{D}^{(\alpha)}
\end{equation*}
where $\mathbf{J}^{(\alpha)} = \frac{\partial y}{\partial w^{(\alpha)}}$ is the Jacobian of the network output with respect to the parameters of layer $\alpha$, $\mathbf{M}$ is the positive definite metric on the output space, and $\mathbf{D}^{(\alpha)}$ is a layer-specific parameter metric (typically diagonal).

\begin{remark}
\label{rem:cholesky}
Using a Cholesky factorisation of $\mathbf{M} = \mathbf{L}^\top \mathbf{L}$ (which exists because $\mathbf{M}$ is positive definite and symmetric), and defining $\mathbf{L}^{(\alpha)} = \mathbf{L} \mathbf{J}^{(\alpha)}$, it becomes evident that the metric at layer $\alpha$ can be written as
\begin{equation*}
\mathbf{G}^{(\alpha)} = {\mathbf{L}^{(\alpha)}}^\top \mathbf{L}^{(\alpha)} + \mathbf{D}^{(\alpha)}.
\end{equation*}
This implies that to recursively track the metric through layers it is sufficient to track the transformation of one of the Cholesky factors. It should be noted that this factorization is only well defined once we have made a choice of ordered basis. This will be the case for submanifolds of a real vector space $\mathbf{R}^n$ with a given ordering.
\end{remark}

This motivates the introduction of a \emph{Riemannian module}, which holds all the necessary information to perform Riemannian gradient descent according to metrics of the kind we just described.

\subsection{Riemannian Modules}

We can also describe the same metric using per-layer-defined data and axioms on how these layers are composed. This is inspired by \cite{large2024scalable}.

\begin{definition}[Riemannian Module]A Riemannian Module is specified by
\begin{itemize}
    \item an input manifold $I$
    \item a parameter manifold $W$ with Riemannian metric $g_W$
    \item an output manifold $O$ with Riemannian metric $g_O$
    \item a smooth map $f \in C^\infty(I \times W, O)$
\end{itemize}
\end{definition}

\begin{definition}[Sequential Composition of Modules]Two Riemannian Modules $M_1$ and $M_2$ can be sequentially composed if their input and output manifolds match, i.e. $I_2 = O_1$ and 
\begin{equation*}
f_2^* g_{O_2} = g_{O_1}
\end{equation*}
The resulting Riemannian module has the obvious definition. 
\end{definition}

\begin{definition}[Parallel Composition of Modules]Two Riemannian Modules $M_1$ and $M_2$ can be composed in parallel using the cartesian product on the components.
\end{definition}

\begin{example}[Linear Layer] Let $I = \mathbf{R^n}, O = \mathbf{R}^m$ and $P = \mathrm{Lin}(\mathbf{R^n},\mathbf{R}^m)$. Let $x \in \mathbf{R}^n$, $y \in \mathbf{R}^m$ and $\mathbf{A}$ a linear map $\mathbf{y} = \mathbf{A} \mathbf{x}$. Let $\mathbf{G}(\mathbf{y})$ be a (non-constant) positive semi-definite Riemannian metric on $\mathbf{O}$, then the pullback of $\mathbf{G}$ along $\mathbf{A}$ to the input space is $\mathbf{A}^\top \mathbf{G} \mathbf{A}$. By Remark \ref{rem:cholesky} we can instead relate $\mathbf{L}'_x = \mathbf{L}_y \mathbf{A}$. The pullback to the space of linear maps is given by
\begin{equation*}
h_\mathbf{x}(\mathbf{A})(\mathbf{B}, \mathbf{C}) = \mathbf{x}^\top \mathbf{B}^\top \mathbf{G}(\mathbf{y}) \mathbf{C} \mathbf{x} = \mathbf{x}^\top \mathbf{B}^\top \mathbf{L}^\top_y \mathbf{L}_y  \mathbf{C} \mathbf{x}
\end{equation*}
\end{example}

\begin{example}[Pointwise Nonlinearity]Let $x \in \mathbf{R}^n$, $y \in \mathbf{R}^n$ and $\sigma \colon \mathbf{R} \to \mathbf{R}$ some nonlinear activation function. Then the matrix of the Jacobian of the map
\begin{equation*}
\phi \colon \mathbf{R}^n \to \mathbf{R}^n, (x_1, \ldots, x_n) \mapsto (\sigma(x_1), \ldots, \sigma(x_n))
\end{equation*}
is the diagonal matrix $\mathbf{J} = \mathbf{diag}(\sigma'(x_1),\ldots, \sigma'(x_n))$. For a given metric $\mathbf{G}(y)$ on the output space, the metric on the input space is then $\mathbf{G}'(x) = \mathbf{J}^\top \mathbf{G}(\mathbf{y})\mathbf{J}$. 
\end{example}

\subsection{Efficient Computation of the Metric Inverse}
\label{ssec:efficient-metric-inverse}
Since in a neural network the number of parameters is often proportional to the square of the number of neurons, an arbitrary metric would be expensive to compute and store explicitly. The layerwise definition in the previous section partially addresses this problem, but the question of how to efficiently invert the metric still stands.

First we observe that by considering a layerwise metric defined as the sum of a weight space-specific metric and the pullback of the output space metric, the inverse can be computed using the Woodbury matrix identity. Using this identity we can then compute the contraction with the pulled back co-vector computed by the backpropagation formula. Assuming that the weight space specific metric is diagonal there is no need to materialize the full inverse metric in memory.

\begin{lemma}[Woodbury Matrix Identity]Let $\mathbf{A} \in \mathbf{R}^{n\times n}$ a square invertible matrix, $\mathbf{U} \in \mathbf{R}^{n \times k}$, $\mathbf{V} \in \mathbf{R}^{k \times n}$ and $\mathbf{B} = \mathbf{A} + \mathbf{U}\mathbf{V}$. If $\mathbf{I}_k + \mathbf{V} \mathbf{A}^{-1} \mathbf{U}$ is invertible, then so is $\mathbf{B}$ and
\begin{equation*}
\mathbf{B}^{-1} =  \mathbf{A}^{-1} - \mathbf{A}^{-1}\mathbf{U} \left(\mathbf{I}_k + \mathbf{V}\mathbf{A}^{-1}\mathbf{U}\right)^{-1} \mathbf{V}\mathbf{A}^{-1}.
\end{equation*}
\end{lemma}

For an arbitrary metric $\mathbf{M}_o(y)$ on the output space, we can indeed use the Cholesky decomposition $\mathbf{M}_o = \mathbf{L}^\top_o \mathbf{L}_o$, where $\mathbf{L}_o$ is upper triangular. The pullback computation then becomes more tractable.

In our case we are interested in the inverse of the metric in each layer. We will assume that it is the sum of a diagonal mass matrix $\mathbf{D}^{(\alpha)}$ and the pullback metric from the output layer. Denoting the Jacobian of the network output with respect to the weights of the current layer as $\mathbf{J}^{(\alpha)}$, we can write the pullback metric as
\begin{equation*}
\mathbf{M}^{(\alpha)}_\mathrm{pb} = {\mathbf{J}^{(\alpha)}}^\top \mathbf{L}^\top_o \mathbf{L}_\mathrm{o} {\mathbf{J}^{(\alpha)}}
\end{equation*}
This matrix will have dimension $n_\alpha \times n_\alpha$ where $n_\alpha$ is the number of parameters in layer $\alpha$. However, the matrix $\mathbf{L}_\mathrm{o} {\mathbf{J}^{(\alpha)}}$ is of dimension $n_o \times n_\alpha$, which can be much smaller if $n_o \ll n_\alpha$. Now we can apply the Woodbury matrix identity to the sum
\begin{equation*}
{\mathbf{M}^{(\alpha)}} = \mathbf{D}^{(\alpha)} + \mathbf{M}^{(\alpha)}_\mathrm{pb} = \mathbf{D}^{(\alpha)} + {\mathbf{J}^{(\alpha)}}^\top \mathbf{L}^\top_o \mathbf{L}_\mathrm{o} {\mathbf{J}^{(\alpha)}}
\end{equation*}
and we get
\begin{equation}
\label{eq:woodbury}
{\mathbf{M}^{(\alpha)}}^{-1} =  {\mathbf{D}^{(\alpha)}}^{-1} - {\mathbf{D}^{(\alpha)}}^{-1} {\mathbf{J}^{(\alpha)}}^\top \mathbf{L}^\top_o \left(\mathbf{I}_{n_o} + \mathbf{L}_\mathrm{o} {\mathbf{J}^{(\alpha)} {\mathbf{D}^{(\alpha)}}^{-1}} {\mathbf{J}^{(\alpha)}}^\top \mathbf{L}^\top_o\right)^{-1} \mathbf{L}_\mathrm{o} {\mathbf{J}^{(\alpha)}}{\mathbf{D}^{(\alpha)}}^{-1}
\end{equation}
In itself, we haven't gained anything, since this matrix will still have $n_\alpha \times n_\alpha$ components. However, we are only interested in right-multiplying it with a vector. The standard gradient is $\nabla_{w^{(\alpha)}} l = {\mathbf{J}^{(\alpha)}}^\top \partial_y l_o$. This yields the weight update equation
\begin{equation*}
\dot{w}^{(\alpha)} = -{\mathbf{M}^{(\alpha)}}^{-1} \nabla_{w^{(\alpha)}} l
\end{equation*}
so that using \eqref{eq:woodbury}, the first term recovers gradient flow scaled by the diagonal metric, while the second term provides the dense geometric correction:
\begin{equation}
\label{eq:gradient-descent-update}
\dot{w}^{(\alpha)} = - {\mathbf{D}^{(\alpha)}}^{-1} \nabla_{w^{(\alpha)}} l + {\mathbf{D}^{(\alpha)}}^{-1} {\mathbf{J}^{(\alpha)}}^\top \mathbf{L}^\top_o \left(\mathbf{I}_{n_o} + \mathbf{L}_\mathrm{o} {\mathbf{J}^{(\alpha)} {\mathbf{D}^{(\alpha)}}^{-1}} {\mathbf{J}^{(\alpha)}}^\top \mathbf{L}^\top_o\right)^{-1} \mathbf{L}_\mathrm{o} {\mathbf{J}^{(\alpha)}}{\mathbf{D}^{(\alpha)}}^{-1} \nabla_{w^{(\alpha)}} l
\end{equation}

\subsection{Algorithm and Complexity Analysis}
\label{ssec:algorithm}

Algorithm~\ref{alg:riemannian-sgd} presents the complete Riemannian SGD procedure with layerwise metrics. The key insight is that we can compute the Riemannian gradient updates efficiently by leveraging the Woodbury identity to avoid explicit inversion of the full layerwise metrics. Note that forming the inner matrix $\mathbf{S}^{(\alpha)}$ restricts all explicit matrix inversions strictly to the dimension of the output space ($n_o \times n_o$).

\begin{algorithm}[h]
\caption{Riemannian SGD with Layerwise Metrics}
\label{alg:riemannian-sgd}
\begin{algorithmic}[1]
\State \textbf{Input:} Neural network $f$, loss function $\ell$, learning rate $\eta$, batch $\{(x_i, y_i)\}$
\State \textbf{Input:} Output metric $\mathbf{M}(y)$, layer mass matrices $\{\mathbf{D}^{(\alpha)}\}$
\State \textbf{Forward pass:} Evaluate network output $y = f(x; w)$
\State \textbf{Initialize:} Compute Cholesky factor $\mathbf{L}_o = \text{chol}(\mathbf{M}(y))$
\For{each layer $\alpha = L, L-1, \ldots, 1$}
    \State Compute global Jacobian $\mathbf{J}^{(\alpha)} = \frac{\partial y}{\partial w^{(\alpha)}}$
    \State Standard gradient: $g^{(\alpha)} = {\mathbf{J}^{(\alpha)}}^\top \frac{\partial \ell}{\partial y}$
    \State Preconditioned gradient: $\mathbf{A}^{(\alpha)} = {\mathbf{D}^{(\alpha)}}^{-1} g^{(\alpha)}$
    \State Scaled Jacobian: $\mathbf{K}^{(\alpha)} = \mathbf{L}_o \mathbf{J}^{(\alpha)}$ \Comment{Dimensions $n_o \times n_\alpha$}
    \State Inner matrix: $\mathbf{S}^{(\alpha)} = \mathbf{I}_{n_o} + \mathbf{K}^{(\alpha)} {\mathbf{D}^{(\alpha)}}^{-1} {\mathbf{K}^{(\alpha)}}^\top$ \Comment{Dimensions $n_o \times n_o$}
    \State Solve system: $\mathbf{v}^{(\alpha)} = (\mathbf{S}^{(\alpha)})^{-1} \mathbf{K}^{(\alpha)} \mathbf{A}^{(\alpha)}$
    \State Update: $\Delta w^{(\alpha)} = -\eta \left(\mathbf{A}^{(\alpha)} - {\mathbf{D}^{(\alpha)}}^{-1} {\mathbf{K}^{(\alpha)}}^\top \mathbf{v}^{(\alpha)}\right)$
    \State $w^{(\alpha)} \leftarrow w^{(\alpha)} + \Delta w^{(\alpha)}$
\EndFor
\end{algorithmic}
\end{algorithm}

\paragraph{Computational Complexity.} The complexity analysis reveals significant efficiency gains over naive approaches:
\begin{itemize}
\item \textbf{Naive metric inversion:} $O(n^3)$ per layer where $n$ is the number of parameters
\item \textbf{Our Woodbury approach:} $O(n \cdot d^2 + d^3)$ per layer where $d = n_o$ is the output dimension
\item \textbf{Memory requirements:} $O(n \cdot d)$ instead of $O(n^2)$ for storing metric information
\end{itemize}

For typical neural networks where $d \ll n$ (e.g., $d = 10$ for CIFAR-10, $d = 1000$ for ImageNet), this represents a substantial computational saving. The total cost per iteration is $O(\sum_\alpha n_\alpha d^2)$ compared to $O(\sum_\alpha n_\alpha^3)$ for full metric computation.

\subsection{Nonlinear Contraction Properties}

Following \cite{kozachkov2023generalization, wensing2020beyond} we can analyze nonlinear contraction properties of the dynamics. In particular we will extend the analysis of section 4.3 in \cite{kozachkov2023generalization} to investigate algorithmic stability. That is we compare training dynamics on a dataset $\mathcal{S}$ and $\mathcal{S}'$, which differs from $\mathcal{S}$ by replacing one sample at index $k$. For ease of comparison we will adopt most of the notation from \cite{kozachkov2023generalization}. Consider a Riemannian module $F_\theta$, parametrized by weights $\theta \in \mathbf{R}^m$ with metric $\mathbf{G}$. For an input vector $\mathbf{x}_i \in \mathbf{R}^k$ the module has a $p$-dimensional output $\mathbf{y}_i \in \mathbf{R}^p$. Following \cite{kozachkov2023generalization} we will denote by $\bar{\mathbf{y}} = [\mathbf{y}_1 \ldots  \mathbf{y}_n]^\top \in \mathbf{R}^{np}$. We will assume that $F_\theta$ is Lipschitz w.r.t to the weights $\theta$ with Lipschitz constant $\kappa$ and that the Jacobian $\nabla_\theta \bar{y}$ (defined in denominator layout as $m \times np$) has full column rank. In that case, since $\mathbf{G}$ is positive definite:
\begin{equation}
\mathbf{H} \ = \ \frac{1}{n} \nabla_\theta^\top \bar{y} \ \mathbf{G}^{-1} \ \nabla_\theta \bar{y}  \ \succeq \  \xi \mathbf{I}
\end{equation}
with $\xi > 0$. This generalized Neural Tangent Kernel (NTK) matrix is the key assumption that enters the following theorem.

\begin{theorem}Using gradient flow with respect to the regularised pullback metric and a mean-squared loss, the Riemannian module $F_\theta$ is algorithmically stable with rate
\begin{equation}
\epsilon_\mathrm{stab} \sim O\left(\frac{\kappa^2 L}{\xi \sqrt{n} \mu}\right)
\end{equation}
where $L$ is the Lipschitz constant of the loss, and $\mu$ is the minimum eigenvalue of the diagonal mass metric $\mathbf{D}$.
\end{theorem}
\begin{proof}
The proof mirrors \cite{kozachkov2023generalization} Theorem 8, except that the gradient with respect to the regularised pullback metric enters. We can consider the change of the output $\mathbf{y}_i$ under the Riemannian gradient descent
\begin{align}
\frac{d}{dt} \mathbf{y}_i &= \nabla_{\theta} F(x_i)^\top \frac{d}{dt} \theta = - \frac{1}{n}  \sum_{j=1}^{n}  \nabla_{\theta} F(x_i)^\top \mathbf{G}^{-1} \nabla_{\theta} l(\mathbf{y}_j, \hat{\mathbf{y}}_j) 
\end{align}
Following Theorem 4 in \cite{kozachkov2023generalization} we can compare this to the dynamics with one sample at index $k$ exchanged
\begin{align}
\frac{d}{dt} \mathbf{y}'_i &= \nabla_{\theta'} F(x_i)^\top \frac{d}{dt} \theta' = - \frac{1}{n} \sum_{j=1}^{n}  \nabla_{\theta'} F(x_i)^\top  \mathbf{G}^{-1} \nabla_{\theta'} l(\mathbf{y}'_j, \hat{\mathbf{y}}_j) + \mathbf{d}_i(t)
\end{align}
with the disturbance given by
\begin{equation}
\mathbf{d}_i(t) = \frac{1}{n} \left[ \nabla_{\theta'} F(x'_k)^\top  \mathbf{G}^{-1} \nabla_{\theta'} l(\mathbf{y}'_k, \hat{\mathbf{y}}'_k) - \nabla_{\theta'} F(x_k)^\top \mathbf{G}^{-1} \nabla_{\theta'} l(\mathbf{y}_k, \hat{\mathbf{y}}_k)\right] 
\end{equation}
with $k$ being the index of the replaced element in $\mathcal{S}'$. Because the metric is defined as $\mathbf{G} = \mathbf{D} + \mathbf{J}^\top \mathbf{M}_o \mathbf{J}$, and the pullback term is positive semi-definite, the eigenvalues of $\mathbf{G}$ are strictly bounded from below by the minimal mass value $\mu$ of the diagonal matrix $\mathbf{D}$. Therefore, $\|\mathbf{G}^{-1}\|_2 \leq \frac{1}{\mu}$. Along with the Lipschitz constant $\kappa$ of the Jacobians, and $L$ of the loss function w.r.t to the network output (meaning $\|\nabla_{\theta} l\| \leq \kappa L$), the norm of the individual disturbance can be upper-bounded as:
\begin{equation}
\lvert \mathbf{d}_i(t) \rvert \leq \frac{2 \kappa^2 L}{n \mu} 
\end{equation}

The time evolution of the stacked vector $\bar{\mathbf{y}}$ is then given as
\begin{equation}
\dot{\bar{\mathbf{y}}} = -\mathbf{H}(t) (\bar{y} - \bar{y}_d)
\end{equation}
with
\begin{equation}
\mathbf{H}(t) = \frac{1}{n} \nabla_{\theta}^\top \bar{y} \mathbf{G}^{-1} \nabla_{\theta} \bar{y} \succeq 0
\end{equation}
where by abuse of notation $\mathbf{G}^{-1}$ also denotes the block diagonal extension to the full parameter space. By construction $\mathbf{G}^{-1}$ is positive definite, therefore $\mathbf{H}(t)$ will be strictly positive definite iff $\nabla_{\theta} \bar{y}$ has full column rank. Then we have for some strictly positive constant $\xi > 0$:
\begin{equation}
\mathbf{H} \succeq \xi \mathbf{I}
\end{equation}
We can bound the stacked disturbance:
\begin{equation}
\lvert \bar{\mathbf{d}} \rvert_2 \leq \sqrt{n} \max_i \lvert \mathbf{d}_i(t) \rvert \leq \frac{2 \kappa^2 L }{\sqrt{n} \mu} 
\end{equation}
We therefore have that after exponential transients of rate $\xi$:
\begin{equation}
\sup_i \lvert \mathbf{y}_i - \mathbf{y}'_i\rvert \leq \frac{2 \kappa^2 L }{\xi \sqrt{n} \mu} 
\end{equation}
\end{proof}

\section{Related Work}
There is a long history of considering Riemannian metrics in the context of neural networks and optimization. Here we focus on work which directly intersects with our approach and position our contributions relative to recent advances.

\paragraph{Singular Geometry and Neural Networks} Semi-Riemannian pullback metrics were introduced in a series of papers \cite{benfenati2022singular, benfenati2023singular, benfenati2024singular}. Here we combine such pullback metrics, which typically are only positive semi-definite, with metrics defined layerwise, yielding positive definite layerwise metrics. This allows us to remain in the realm of conventional Riemannian geometry, and also to use the metric to compute a gradient direction. Earlier work recognized that neural networks are ``singular learning machines'' \cite{watanabe2009algebraic}, motivating the need for geometric approaches that can handle singular spaces.

\paragraph{Action Principles for Gradient Descent} The action principle we employ, which frames gradient descent as the critical point of a Lagrangian, has roots in theoretical physics. Specifically, it corresponds to the bosonic part of the Euclidean supersymmetric sigma model introduced by Witten \cite{witten1982supersymmetry}. Our contribution extends this idea to modular systems, leveraging the layered structure of neural networks to derive backpropagation as a constrained optimization problem. This is distinct from the Bregman-Lagrangian approach, which provides variational formulations of accelerated optimization methods \cite{wibisono2016variational}, and from recent work on Hamiltonian neural networks that preserves physical structure in the learned dynamics.

\paragraph{Natural Gradients and Approximations} The natural gradient, introduced by Amari \cite{amari1998natural}, accounts for the geometry of the parameter space by using the Fisher information matrix as a Riemannian metric. While effective, computing the Fisher matrix scales as $O(p^2)$ where $p$ is the number of parameters, leading to approximations like K-FAC \cite{martens2015optimizing}, PRONG \cite{desjardins2015natural}, and more recent approaches like Shampoo. Unlike these global methods, our layerwise Riemannian metric is recursively defined and computed, exploiting the modular structure of neural networks to achieve better scaling properties. Recent work on neural tangent kernels provides related geometric insights, but focuses on the infinite-width limit rather than finite networks.

\paragraph{Modular and Compositional Optimization} Our Riemannian module framework relates to recent work on compositional optimization and modular architectures. The concept of modules that can be composed sequentially or in parallel resonates with work on neural module networks and compositional generalization. However, our focus is on the geometric structure of the optimization process rather than architectural modularity. This connects to recent advances in understanding optimization landscapes of deep networks, including work on loss surface geometry and mode connectivity.

\paragraph{Nonlinear Contraction Theory and Optimization} Nonlinear contraction theory \cite{lohmiller1998contraction} provides a framework for analyzing the stability of nonlinear dynamical systems. Its application to neural network optimization has been explored in works like \cite{kozachkov2023generalization, wensing2020beyond}, which analyze the stability and convergence of gradient-based training. Our work extends these results by applying contraction theory to the dynamics induced by our layerwise Riemannian metric, providing algorithmic stability guarantees. This complements recent theoretical advances in understanding generalization through the lens of algorithmic stability and uniform convergence.

\paragraph{Efficient Second-Order Methods} Our Woodbury-based approach for efficient metric inversion relates to the broader literature on second-order optimization methods. Recent work has focused on making Newton-type methods practical for deep learning through various approximations and computational tricks. Our contribution lies in providing a principled geometric framework that naturally leads to efficient computational strategies, rather than starting from computational considerations and approximating the geometry. 

\section{Limitations}

\paragraph{Computational and Memory Overhead.} While the Woodbury identity significantly reduces the computational cost from $O(n^3)$ to $O(n \cdot d^2 + d^3)$ per layer, our approach still incurs overhead compared to standard SGD. The method requires: (1) computing and storing Jacobians $\mathbf{J}^{(\alpha)}$ for each layer, (2) maintaining Cholesky factors $\mathbf{L}_o$ of the output metric, and (3) solving the inner linear systems in equation~\eqref{eq:gradient-descent-update}. For networks where the output dimension $d$ is not small compared to parameter counts $n_\alpha$, this overhead may become prohibitive.

\paragraph{Choice of Output Space Metric.} Our approach requires specifying a metric $\mathbf{M}(y)$ on the output space. While we explored identity and Hessian-based choices, the optimal selection remains problem-dependent and may require domain expertise. The Gauss-Newton approximation to the Hessian, while computationally tractable, may not capture the full curvature information of the loss landscape.

\paragraph{Theoretical Scope.} For simplicity, we derived our results for neural networks where activations are defined in vector spaces. A more general treatment for arbitrary manifolds (e.g., networks operating on graphs or other structured spaces) would require extending our framework. Additionally, our contraction analysis assumes certain regularity conditions (Lipschitz continuity, full row Jacobians) that may not hold throughout training.

\paragraph{Limited Empirical Validation.} Our experiments are limited to image classification on MNIST and CIFAR-10. The benefits of our geometric approach may not generalize to other domains (NLP, reinforcement learning, generative modeling) or to tasks where the inductive biases captured by our metrics are less relevant.

\paragraph{Hyperparameter Sensitivity.} The method introduces additional hyperparameters (diagonal mass values $\mathbf{D}^{(\alpha)}$, regularization parameters) that require tuning. The sensitivity to these choices and their interaction with standard hyperparameters (learning rate, batch size) needs further investigation.

\section{Discussion}
Our work provides a novel perspective on optimizing modular systems, such as neural networks. By reframing backpropagation as the critical point of an action on a Riemannian manifold, we recover the algorithm in a principled manner and introduce a layerwise Riemannian metric that respects the modular structure of such systems. This approach opens new avenues for designing optimization algorithms that use the intrinsic geometry of modular systems.

The recursively defined layerwise Riemannian metric is a key contribution of this work. Unlike global metrics, such as the Fisher information metric used in natural gradient descent, our metric is computed layer by layer, making it computationally efficient and naturally aligned with the modular nature of neural networks. The use of the Woodbury matrix identity to invert this metric further enhances its practicality.

The concept of Riemannian modules introduced here provides a framework for analyzing and optimizing modular systems beyond neural networks. By defining modules with input, output, and parameter manifolds equipped with Riemannian metrics, we enable the composition of such modules in sequential or parallel configurations. This modularity mirrors biological systems, where evolution and development optimize interconnected components, and engineering systems, where modular design is a cornerstone of scalability. The ability to quantify convergence properties using nonlinear contraction theory further strengthens the applicability of our framework. We partially validated this here by deriving guarantees on algorithmic stability.

Beyond machine learning, applying this framework to biological or engineering systems—where modularity and optimization over time are prevalent—offers exciting opportunities. For instance, modeling developmental processes in biology or optimizing modular engineered systems could benefit from the principles developed here.

In conclusion: By grounding backpropagation in a variational principle and introducing a layerwise Riemannian metric, we provide both a deeper understanding of existing methods and a foundation for future advances in understanding optimization in modular systems.

\bibliographystyle{abbrv}
\bibliography{bibliography}

@misc{benfenati2022singular,
      title={A singular Riemannian geometry approach to Deep Neural Networks I. Theoretical foundations},
      author={Alessandro Benfenati and Alessio Marta},
      year={2022},
      eprint={2201.09656},
      archivePrefix={arXiv},
      primaryClass={cs.LG}
}

@article{benfenati2024singular,
  title={A singular Riemannian Geometry Approach to Deep Neural Networks III. Piecewise Differentiable Layers and Random Walks on $ n $-dimensional Classes},
  author={Benfenati, Alessandro and Marta, Alessio},
  journal={arXiv preprint arXiv:2404.06104},
  year={2024}
}

@article{benfenati2023singular,
  title={A singular Riemannian geometry approach to deep neural networks II. Reconstruction of 1-D equivalence classes},
  author={Benfenati, Alessandro and Marta, Alessio},
  journal={Neural Networks},
  volume={158},
  pages={344--358},
  year={2023},
  publisher={Elsevier}
}

@article{wensing2020beyond,
  title={Beyond convexity—contraction and global convergence of gradient descent},
  author={Wensing, Patrick M and Slotine, Jean-Jacques},
  journal={Plos one},
  volume={15},
  number={8},
  pages={e0236661},
  year={2020},
  publisher={Public Library of Science San Francisco, CA USA}
}

@article{
kozachkov2023generalization,
title={Generalization as Dynamical Robustness--The Role of Riemannian Contraction in Supervised Learning},
author={Leo Kozachkov and Patrick Wensing and Jean-Jacques Slotine},
journal={Transactions on Machine Learning Research},
issn={2835-8856},
year={2023},
url={https://openreview.net/forum?id=Sb6p5mcefw},
note={}
}

@book{watanabe2009algebraic,
  title={Algebraic geometry and statistical learning theory},
  author={Watanabe, Sumio},
  volume={25},
  year={2009},
  publisher={Cambridge university press}
}

@article{amari1998natural,
  title={Natural gradient works efficiently in learning},
  author={Amari, Shun-Ichi},
  journal={Neural computation},
  volume={10},
  number={2},
  pages={251--276},
  year={1998},
  publisher={MIT Press}
}

@article{lohmiller1998contraction,
  title={On contraction analysis for non-linear systems},
  author={Lohmiller, Winfried and Slotine, Jean-Jacques E},
  journal={Automatica},
  volume={34},
  number={6},
  pages={683--696},
  year={1998},
  publisher={Elsevier}
}

@article{wibisono2016variational,
  title={A variational perspective on accelerated methods in optimization},
  author={Wibisono, Andre and Wilson, Ashia C and Jordan, Michael I},
  journal={proceedings of the National Academy of Sciences},
  volume={113},
  number={47},
  pages={E7351--E7358},
  year={2016},
  publisher={National Academy of Sciences}
}

@article{desjardins2015natural,
  title={Natural neural networks},
  author={Desjardins, Guillaume and Simonyan, Karen and Pascanu, Razvan and others},
  journal={Advances in neural information processing systems},
  volume={28},
  year={2015}
}

@inproceedings{martens2015optimizing,
  title={Optimizing neural networks with kronecker-factored approximate curvature},
  author={Martens, James and Grosse, Roger},
  booktitle={International conference on machine learning},
  pages={2408--2417},
  year={2015},
  organization={PMLR}
}

@article{large2024scalable,
  title={Scalable optimization in the modular norm},
  author={Large, Tim and Liu, Yang and Huh, Jacob and Bahng, Hyojin and Isola, Phillip and Bernstein, Jeremy},
  journal={Advances in Neural Information Processing Systems},
  volume={37},
  pages={73501--73548},
  year={2024}
}

@incollection{lecun1988theoretical,
  title={A theoretical framework for back-propagation},
  author={Lecun, Yann},
  booktitle={Proceedings of the 1988 Connectionist Models Summer School, CMU, Pittsburg, PA},
  pages={21--28},
  year={1988},
  publisher={Morgan Kaufmann}
}

@article{witten1982supersymmetry,
  title={Supersymmetry and Morse theory},
  author={Witten, Edward},
  journal={Journal of differential geometry},
  volume={17},
  number={4},
  pages={661--692},
  year={1982},
  publisher={Lehigh University}
}

@article{bryson1969applied,
  title={Applied optimal control. 1969},
  author={Bryson, AE and Ho, Yu-Chi},
  journal={Blaisdell, Waltham, Mass},
  volume={8},
  number={72},
  pages={14},
  year={1969}
}

\end{document}